%% file: main.tex
\newif\ifRAL
\RALfalse 	

\newif\ifTRO
\TROfalse 	

\ifRAL
    \documentclass[letterpaper, 10 pt, journal, twoside]{IEEEtran}
\else
    \ifTRO
        \documentclass[lettersize,journal]{IEEEtran}
    \else
        \documentclass[letterpaper, 10 pt, conference]{IEEEconf}
    \fi
\fi

\usepackage{balance}
\usepackage{algorithm}
\usepackage{enumerate}
\usepackage{amsmath} 
\usepackage{amssymb}  
\usepackage{graphicx} 
\usepackage{amsfonts}
\usepackage{xcolor,microtype}
\usepackage{cite}
\usepackage[font=small,skip=1pt]{subcaption} 
\usepackage[font=small,skip=1pt]{caption}
\usepackage{url}
\usepackage{booktabs}
\usepackage{bm}

\usepackage{soul}  
\usepackage{comment} 
\usepackage{mathtools} 

\usepackage{todonotes}

\usepackage[
    colorlinks=false,
    linkbordercolor=red
]{hyperref}

\usepackage{amsthm} 
\usepackage[utf8]{inputenc}
\usepackage[english]{babel}

\usepackage{siunitx}    

\newtheorem{proposition}{Proposition}
\newtheorem{definition}{Definition}[section]

\newtheorem{assumption}{Assumption}

\input{definitions.tex}

\title{\LARGE \bf 
Vision-based Underwater Formation Control \\ With Input Saturations
via Barrier Lyapunov Functions 
}

\author{
Nicola De Carli, João Zenário, Victor Nan Fernandez-Ayala, Dimos V. Dimarogonas
}

\begin{document}

\maketitle
\thispagestyle{empty}
\pagestyle{empty}

\begin{abstract}
\input{Sections/Abstract}
\end{abstract}

\ifRAL
\fi

\section{Introduction}
\input{Sections/introduction}

\label{sec:introduction}

\section{Preliminaries and Problem Formulation}
\input{Sections/problem_formulation}

\label{sec:preliminaries}

\section{Control Design}
\input{Sections/control}

\label{sec:control}

\section{Simulation 
Results}
\input{Sections/results}

\label{sec:results}


\section{Conclusions}
\input{Sections/conclusions}
\label{sec:conclusions}



\bibliographystyle{IEEEtran}
\bibliography{references}

\end{document}

%% file: definitions.tex
\DeclareMathOperator*{\diag}{diag}

\DeclareMathOperator*{\argmin}{\arg\!\min}

\DeclareMathOperator*{\col}{col}

\newcommand{\nR}[1]{\mathbb{R}^{#1}}		
\newcommand{\nS}[1]{\mathbb{S}^{#1}}		

\newcommand{\vect}[1]{\ensuremath{\bm{#1}}}		
\newcommand{\matr}[1]{\ensuremath{\bm{#1}}}		

\newcommand{\define}{\coloneqq}			
\newcommand{\norm}[1]{\left\lVert#1\right\rVert}

\newcommand{\setparenthesis}[1]{\left\{#1\right\}}

\newcommand{\skewM}[1]{[{#1}]_{\times}}

\newcommand{\eye}[1]{\matr{I}_{#1}}
\newcommand{\zeros}[1]{\matr{0}_{#1}}

\renewcommand{\frame}{\mathcal{F}}		

\newcommand{\pos}{\vect{p}}				
\newcommand{\vel}{\vect{v}}				
\newcommand{\quat}{\boldsymbol{\mathfrak{q}}}				
\newcommand{\angvel}{\vect{\omega}}				
\newcommand{\rotMat}{\matr{R}}				

\newcommand{\pose}{\vect{\eta}}
\newcommand{\dpose}{\dot{\vect{\eta}}}
\newcommand{\veltwist}{\vect{\nu}}
\newcommand{\dveltwist}{\dot{\vect{\nu}}}

\newcommand{\inertiaM}{\vect{M}}
\newcommand{\coriolis}{\vect{C}}
\newcommand{\damping}{\vect{D}}

\newcommand{\wrench}{\mathbf{w}}
\newcommand{\gravitywrench}{\vect{g}}

\newcommand{\wrenchSet}{\mathcal{W}}

\newcommand{\Nrob}{\ensuremath{N}}
\newcommand{\graph}{\ensuremath{\mathcal{G}}}
\newcommand{\edges}{\ensuremath{\mathcal{E}}}
\newcommand{\vertices}{\ensuremath{\mathcal{V}}}

\newcommand{\neigh}{\ensuremath{\mathcal{N}}}

\newcommand{\followerset}{\ensuremath{\mathcal{F}}}

\newcommand{\recenteredBarrier}{V}

\newcommand{\todoing}[1]{\todo[inline,color=blue!20, linecolor=orange!250]{\small#1}}

\newcommand{\nicksay}[1]{\todoing{\textbf{Nick:} #1}}

%% file: Sections/abstract.tex
In this work, we propose a communication-free framework for vision-based formation control of fully actuated underwater robots subject to sensing constraints, collision-avoidance requirements, and input saturations. Recentered barrier Lyapunov functions encode sensing and collision-avoidance constraints, while command-filtered backstepping extends the design to the second-order vehicle dynamics. The resulting control objective is enforced through a quadratic program that explicitly accounts for actuator limits. Conservative sensing domains provide margins from the physical limits and are adaptively relaxed when necessary, allowing temporary violation of the conservative bounds. The proposed approach is validated through realistic Software-in-the-Loop (SITL) simulations in Gazebo. \href{https://github.com/KTH-DHSG/brov2_constrained_formation_control}{[Code]} \href{https://www.youtube.com/watch?v=n3sPo5NFHVQ}{[Video]} 


%% file: Sections/introduction.tex
Underwater robotics remains particularly challenging due to uncertain vehicle dynamics and severe limitations in perception and communication \cite{torroba2026marinarium, zhou2021survey, basso2025terrain}. In particular, underwater vehicles are subject to complex and often poorly predictable hydrodynamic forces \cite{fossen2011handbook}, while GPS is unavailable below the surface and acoustic communication is typically limited in range, bandwidth, and reliability. These limitations become even more critical in multi-robot systems, where coordination requires either explicit communication or implicit information exchange through onboard sensing. As a result, the deployment of cooperative underwater robotic teams remains comparatively limited \cite{zhou2021survey, yan2023formation, matouvs2022formation, hoff2024communication}.

{
A fundamental capability for multi-robot coordination is formation control, namely the ability of a team to maintain a prescribed geometric configuration. In sufficiently clear water, onboard cameras can provide relative information without requiring explicit communication. However, limited field of view and sensing range, the latter being particularly restrictive underwater, must be explicitly accounted for to preserve visibility of neighboring robots. Reactive approaches based on control barrier functions (CBFs) \cite{de2024distributed, mestres2024distributed} and barrier Lyapunov functions (BLFs) \cite{panagou2015distributed, restrepo2022robust, csekerciouglu2024robust} have been widely used for this purpose. In this work, we focus on BLFs, whose divergence near the constraint boundary provides a strong repulsive action, but also makes them undefined once the constraint is violated. This becomes problematic when conflicting objectives or actuator limitations make temporary violation of conservative bounds unavoidable.

Related issues arise in prescribed performance control \cite{mehdifar2022funnel, mehdifar2025robust, trakas2023robust}, where auxiliary dynamics have been introduced to temporarily relax selected constraints and recover their nominal bounds once feasibility is restored \cite{mehdifar2022funnel, mehdifar2025robust, trakas2023robust, restrepo2024tracking}. Existing approaches either consider box constraints \cite{mehdifar2022funnel, trakas2023robust} or aggregate multiple hard and soft constraints through smooth minimum operators and a common relaxation variable \cite{mehdifar2025robust, restrepo2024tracking}.
}

In this work, we propose a formation control framework for fully actuated underwater vehicles described by six-degree-of-freedom Euler--Lagrange-type dynamics and interacting over a directed sensing graph. Each follower relies only on relative position measurements from its parent, so no explicit inter-robot communication is required. Formation objectives, collision avoidance, sensing range, and field-of-view constraints are encoded at the pose level through recentered barrier Lyapunov functions. The resulting design is extended to the vehicle dynamics through command-filtered backstepping, yielding a control Lyapunov function (CLF) whose derivative is affine in the thruster forces. A quadratic program (QP) then enforces the desired dissipation associated with the generalized-velocity tracking error directly in thruster coordinates, while explicitly accounting for actuator saturations. Conservative barrier domains are combined with an adaptive relaxation mechanism that enlarges an individual domain only when its margin becomes small and the zero-relaxation CLF condition is incompatible with the actuator limits, and subsequently recovers it toward its nominal value. 
Unlike existing relaxation schemes, the proposed approach is not restricted to box constraints and adapts each constraint independently according to its current margin. The framework is validated through Software-in-the-Loop simulations in Gazebo with PX4 and BlueROV2 Heavy underwater vehicles. 

The remainder of the paper is organized as follows. Section II introduces the system model and problem formulation, Section III presents the control design, Section IV reports simulation results, and Section V concludes the paper.

%% file: Sections/problem_formulation.tex
\subsection*{Notation}

The sets
\(
\nR{}
\) and
\(
\nR{}_{\geq 0},
\)
denote the real numbers and the nonnegative real numbers. 
Bold lowercase symbols denote column vectors, bold uppercase symbols denote matrices, and calligraphic uppercase symbols denote sets. For a set \(\mathcal{S}\), its boundary is denoted by
\(
\partial\mathcal{S}.
\) The \(n\times n\) identity matrix and the \(m\times n\) zero matrix are denoted by
\(
\eye{n}
\)
and
\(
\mathbf{0}_{m\times n},
\)
respectively.
For a square matrix
\(
\matr{A}\in\nR{n\times n},
\)
its trace is denoted by
\(
\operatorname{tr}(\matr{A})
\).
For vectors
\(
\vect{x}_1,\ldots,\vect{x}_m,
\)
their vertical concatenation is denoted by
\(\col(\vect{x}_1,\ldots,\vect{x}_m) \).
The Euclidean norm is denoted by
\(
\norm{\vect{x}},
\)
the distance of a point \(\vect{x}\) from a set \(\mathcal{S}\) by
\(
\lvert\vect{x}\rvert_{\mathcal{S}}
\define
\inf_{\vect{y}\in\mathcal{S}}
\norm{\vect{x}-\vect{y}},
\)
while, for a positive-definite matrix
\(
\matr{P},
\)
the corresponding weighted norm is
\(
\norm{\vect{x}}_{\matr{P}}
\define
\sqrt{
\vect{x}^{\top}
\matr{P}
\vect{x}
}.
\)
The unit three-sphere and the special orthogonal group are defined as
\(\nS{3}
    \define
    \left\{
        \vect{q}\in\nR{4}
        \,\middle|\,
        \norm{\vect{q}}=1
    \right\}
\) and
\(
    \mathrm{SO}(3)
    \define
    \left\{
        \matr{R}\in\nR{3\times3}
        \,\middle|\,
        \matr{R}^{\top}\matr{R}=\eye{3},
        \ \det(\matr{R})=1
    \right\}
\).
For
\(
\matr{A}\in\nR{3\times3},
\)
its skew-symmetric part is denoted by
\(
    \operatorname{skew}(\matr{A})
    \define
    \frac{1}{2}
    \left(
        \matr{A}-\matr{A}^{\top}
    \right).
\)
For \(\vect x\in\mathbb R^3\), let \([\vect x]_\times\vect y
=\vect x\times\vect y\), with \(([\vect x]_\times)^\vee=\vect x\).
%
For \(x\in\mathbb{R}\), let
\(
[x]_+\define\max\{x,0\}.
\)


\subsection{System Model}
\label{subsec:model}

Consider a team of \(\Nrob\) fully-actuated underwater vehicles. For each robot \(i\in \vertices
\define\setparenthesis{1,\ldots,\Nrob}\), let \(\frame_I\), \(\frame_{B_i}\), and \(\frame_{C_i}\) denote, respectively, the inertial frame, the body-fixed frame, and the camera frame. The pose of robot \(i\) is described by
\(
\pose_i
\define
(
\pos_i,
\quat_i
)
\in
\nR{3}\times\nS{3},
\)
where \(\pos_i\in\nR{3}\) is the position of the body-frame origin expressed in \(\frame_I\), and \(\quat_i\in\nS{3}\) is the unit quaternion representing the orientation of \(\frame_{B_i}\) with respect to \(\frame_I\). The corresponding body-fixed generalized velocity is
\(
\veltwist_i
\define
\col(
\vel_i,
\angvel_i
)
\in\nR{6},
\)
where \(\vel_i\in\nR{3}\) and \(\angvel_i\in\nR{3}\) are the linear and angular velocities expressed in \(\frame_{B_i}\), respectively.

Following the standard six-degree-of-freedom marine-craft model in~\cite{fossen2011handbook}, the dynamics of each robot \(i\in \vertices\) are written as
\begin{subequations}
\label{eq:underwater_vehicle_model}
\begin{align}
    \dpose_i
    &=
    \boldsymbol{J}(\pose_i)\veltwist_i,
    \label{eq:underwater_vehicle_model_kinematics}
    \\
    \dveltwist_i
    &=
    \inertiaM_i^{-1}
    \left(
        \wrench_i
        -
        \coriolis_i(\veltwist_i)\veltwist_i
        -
        \damping_i(\veltwist_i)\veltwist_i
        -
        \gravitywrench_i(\pose_i)
    \right).
    \label{eq:underwater_vehicle_model_dynamics}
\end{align}
\end{subequations}
The kinematic transformation
\(
\boldsymbol{J}(\pose_i)\in\nR{7\times6}
\)
maps the body-fixed generalized velocity into the time derivative of the
pose coordinates. Let
\(
\rotMat_i\define\rotMat(\quat_i)\in \mathrm{SO}(3)
\)
denote the rotation matrix from \(\frame_{B_i}\) to \(\frame_I\), and,
under the scalar-first convention, write
\(
\quat_i=\col(q_{0,i},\vect{q}_{v,i})\in\nS{3}
\).
Then,
\begin{equation*}
    \boldsymbol{J}(\pose_i)
    \define
    \begin{bmatrix}
        \rotMat_i & \boldsymbol{0}_{3\times3}
        \\
        \boldsymbol{0}_{4\times3} &
        \frac{1}{2}\boldsymbol{Q}(\quat_i)
    \end{bmatrix},
    \quad
    \boldsymbol{Q}(\quat_i)
    \define
    \begin{bmatrix}
        -\vect{q}_{v,i}^{\top}
        \\
        q_{0,i}\eye{3}+\skewM{\vect{q}_{v,i}}
    \end{bmatrix}.
    \label{eq:kinematic_transformation}
\end{equation*}
The constant matrix  \( \inertiaM_i\in\nR{6\times 6} \) is symmetric and positive definite and represents the generalized inertia matrix, including the rigid-body and added-mass contributions,  \( \coriolis_i(\veltwist_i)\in\nR{6\times 6} \) collects the corresponding Coriolis and centripetal terms,  \( \damping_i(\veltwist_i)\in\nR{6\times 6} \) models the hydrodynamic damping, and \( \gravitywrench_i(\pose_i)\in\nR{6} \)  contains the gravitational and buoyancy restoring forces and moments. Furthermore,  
\( \wrench_i\in\nR{6} \) denotes the control wrench expressed in \(\frame_{B_i}\). 
The generalized wrench is generated by $n_{f,i}>6$ thrusters. Let
\(
    \vect{f}_i
    \define
    \col(
        f_{i,k}
    )_{k=1}^{n_{f,i}}
    \in\nR{n_{f,i}}
\)
collect the individual thruster forces. The control-allocation model is
\(
    \wrench_i
    =
    \boldsymbol{B}_i\vect{f}_i,
\)
where
\(
\boldsymbol{B}_i\in\nR{6\times n_{f,i}}
\)
is the control-allocation matrix determined by the positions and thrust directions of the actuators. We assume that
\(
\operatorname{rank}\left(\boldsymbol{B}_i\right)=6,
\)
so that \(\boldsymbol{B}_i\) has full row rank and the vehicle is fully actuated. 
The actuator limits are represented directly by the box constraint
\begin{equation}\label{eq:input_lims}
    \underline{\vect{f}}_i
        \leq
        \vect{f}_i
        \leq
        \overline{\vect{f}}_i
\end{equation}
where
\(
\underline{\vect{f}}_i,\overline{\vect{f}}_i\in\mathbb{R}^{n_{f,i}}
\)
are the componentwise lower and upper thruster-force bounds.

\subsubsection*{\textbf{Camera Model}}
Each robot is equipped with a rigidly attached forward-looking depth camera. The positive \(x\)-axis of the camera frame \(\frame_{C_i}\) is assumed to coincide with the optical axis, while its \(y\)- and \(z\)-axes define the horizontal and vertical image directions, respectively. Let
\(
\rotMat_{CBi}\in \mathrm{SO}(3)
\)
denote the rotation matrix from \(\frame_{B_i}\) to \(\frame_{C_i}\), and let
\(
\pos_{CBi}\in\nR{3}
\)
denote the position of the camera-frame origin relative to the body-frame origin, expressed in \(\frame_{B_i}\). These extrinsic parameters are assumed constant and known from calibration.

The relative position of robot \(j\) with respect to robot \(i\), expressed in the inertial frame, is
\(
    \pos_{ij}
    \define
    \pos_j-\pos_i
\), and the relative distance is denoted by \(d_{ij} \define \norm{\pos_{ij}}\).
The corresponding coordinates in the body and camera frames of robot \(i\) are, respectively, 
\(
    \pos_{ij}^{B}
    =
    \rotMat_i^\top\pos_{ij},
\) and
\begin{equation}\label{eq:relative_target_camera}
    \pos_{ij}^{C}
    =
    \rotMat_{CBi}
    \left(
        \pos_{ij}^{B}-\pos_{CBi}
    \right).
\end{equation}
Writing
\(
    \pos_{ij}^{C}
    =
    \col(
        p_{x,ij}^{C},
        p_{y,ij}^{C},
        p_{z,ij}^{C}
    ),
\)
the camera is characterized by horizontal and vertical half-aperture angles
\(
\theta_h^{\max},\theta_v^{\max}\in(0,\pi/2)
\)
and by a maximum sensing range
\(
d_{\max}>0.
\)
Accordingly, a target can be detected only if it lies in front of the camera, within its angular field of view, and within the admissible sensing range, namely,
\begin{equation}
\begin{alignedat}{2}
    d_{ij}
    &< d_{\max},
    \qquad&
    p_{x,ij}^{C}
    &> 0,
    \\
    \left|p_{y,ij}^{C}\right|
    &<
    p_{x,ij}^{C}\tan\left(\theta_h^{\max}\right),
    \qquad&
    \left|p_{z,ij}^{C}\right|
    &<
    p_{x,ij}^{C}\tan\left(\theta_v^{\max}\right).
\end{alignedat}
\label{eq:camera_sensing_limits}
\end{equation}
The maximum range \(d_{\max}\) depends on the sensing hardware and on the operating conditions and is particularly relevant underwater, where light attenuation and scattering can significantly reduce visibility as the water turbidity increases. For simplicity, \(d_{\max}\) is defined here in terms of the inter-robot distance \(\left\|\pos_{ij}\right\|\), rather than the distance between the camera center of robot \(i\) and the observed point on robot \(j\).

To express the field of view constraints in normalized form, define the horizontal and vertical image coordinates
\begin{equation}
    \alpha_{h,ij}
    \define
    \frac{p_{y,ij}^{C}}
    {p_{x,ij}^{C}\tan\left(\theta_h^{\max}\right)},
    \qquad
    \alpha_{v,ij}
    \define
    \frac{p_{z,ij}^{C}}
    {p_{x,ij}^{C}\tan\left(\theta_v^{\max}\right)}.
    \label{eq:normalized_image_coordinates}
\end{equation}
The conditions in \eqref{eq:camera_sensing_limits} are therefore equivalently written as (see also Fig.~\ref{fig:sensing_geometry})
\begin{equation}
\begin{alignedat}{2}
    &
    d_{ij}
    < d_{\max},
    \quad
    &p_{x,ij}^{C}
    > 0,
    \quad
    &\left|\alpha_{h,ij}\right| < 1,
    \quad
    &\left|\alpha_{v,ij}\right| < 1.
\end{alignedat}
\label{eq:camera_sensing_limits_alpha}
\end{equation}
Because \(\pos_{ij}^{C}\) depends on the relative position \( \pos_{ij} \) and on the orientation of the observing robot, visibility preservation is intrinsically coupled to its translational and rotational motion.

\begin{figure}[t]
    \centering
    \includegraphics[width=0.40\columnwidth]
    {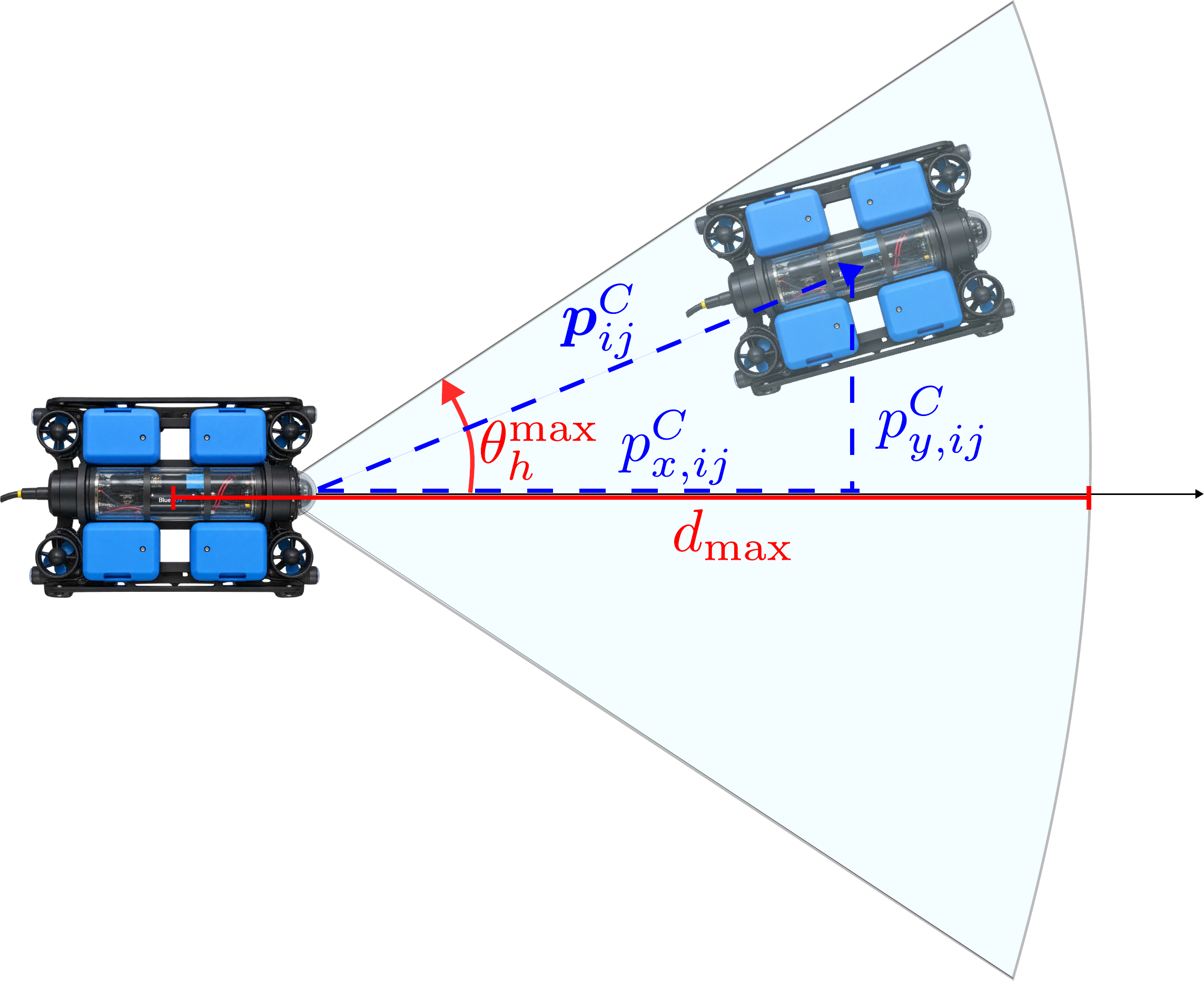}
    \caption{
       Geometric interpretation of the sensing constraints.
        Top view illustration of the maximum sensing range
        \(d_{\max}\) and camera field of view. 
    }
    \label{fig:sensing_geometry}
\end{figure}


\subsection{{Interaction Graph}}
\label{subsec:graph}

Underwater communication is typically characterized by limited bandwidth, significant latency, and reduced reliability. Accordingly, the proposed coordination architecture does not rely on the exchange of information among the robots. Instead, each vehicle uses only its own state and relative measurements of neighboring vehicles obtained through its onboard camera.

The sensing interactions among the robots are described by a directed graph
\(
\graph
\define
\left(
    \vertices,
    \edges
\right),
\)
where
\(
\edges\subseteq\vertices\times\vertices
\)
is the set of directed sensing edges. A directed edge
\(
e_{ij}=(i,j)\in\edges
\)
indicates that robot \(i\) observes robot \(j\) through its onboard camera. The out-neighbor set of robot \(i\) is therefore defined as
\(
    \neigh_i
    \define
    \left\{
        j\in\vertices
        \,\middle|\,
        (i,j)\in\edges
    \right\}.
\)
The sensing graph being directed is a natural consequence of the field-of-view limitations: the fact that robot \(i\) can sense robot \(j\) does not necessarily imply that robot \(j\) can sense robot \(i\).

We consider a leader--follower sensing topology in which each follower
observes a unique parent.

\begin{assumption}
\label{ass:sensing_graph}
The sensing graph \(\graph\) is a directed spanning tree rooted at the
leader~\cite[Section~3.3]{bullo2022network}.
\end{assumption}



\subsection{Problem Statement}
\label{subsec:problem}

Let robot \(1\) be the leader, with desired position trajectory
\(
\pos_1^{d}:\nR{}_{\geq0}\rightarrow\nR{3}.
\)
The reference
\(
\pos_1^{d}(t)
\)
and its derivatives
are assumed available to the leader. Each follower
\(
i\in \followerset \define \vertices\setminus\{1\}
\)
is assigned a desired relative displacement
\(
\pos_{ij}^{d}\in\nR{3}
\)
with respect to its unique parent
\(
j\in\neigh_i.
\)
Define the leader tracking error and the relative formation errors as
\begin{equation}
    \tilde{\pos}_{1}
    \define
    \pos_1-\pos_1^{d},
    \qquad
    \tilde{\pos}_{ij}
    \define
    \pos_{ij}-\pos_{ij}^{d},
    \quad
    (i,j)\in\edges.
    \label{eq:formation_errors}
\end{equation}

Let
\(
d_{\min}>0
\)
denote the minimum admissible inter-robot distance. For each sensing edge
\(
(i,j)\in\edges,
\)
define the actual \emph{physical} admissible relative-pose domain as
\begin{equation}
    \mathcal{D}_{ij}
    \define
    \setparenthesis{
        \left(\pos_{ij},\quat_i\right)
        \ \middle|\
        \begin{aligned}
            d_{\min}
            &<
            d_{ij}
            <
            d_{\max},
            \qquad
            p_{x,ij}^{C}
            >
            0,
            \\
            \left|\alpha_{h,ij}\right|
            &<
            1,
            \qquad
            \left|\alpha_{v,ij}\right|
            <
            1
        \end{aligned}
    }.
    \label{eq:actual_admissible_domain}
\end{equation}

The nominal controller is designed over a conservative domain strictly contained in \(\mathcal{D}_{ij}\). Let
\begin{align}
    d_{\min}^{c}>d_{\min},
    \quad
    d_{\max}^{c}<d_{\max},
    \quad
    0<\alpha_h^{c}<1,
    \quad
    0<\alpha_v^{c}<1
    \label{eq:conservative_domain_parameters}
\end{align}
denote its boundaries, and define the corresponding set $\mathcal{D}_{ij}^{c}$.
By construction,
\(
\mathcal{D}_{ij}^{c}\subset\mathcal{D}_{ij}.
\)


\begin{assumption}
\label{ass:desired_formation_feasible}
For every sensing edge
\(
(i,j)\in\edges,
\)
the desired relative pose lies strictly inside
\(
\mathcal{D}_{ij}^{c}.
\)
Equivalently,
\begin{equation}
\begin{alignedat}{2}
    d_{\min}^{c}
    <
    d_{ij}^d
    &<
    d_{\max}^{c},
    \qquad&
    p_{x,ij}^{C,d}
    &>
    0,
    \\
    \left|\alpha_{h,ij}^{d}\right|
    &<
    \alpha_h^{c},
    \qquad&
    \left|\alpha_{v,ij}^{d}\right|
    &<
    \alpha_v^{c}.
\end{alignedat}
\label{eq:desired_formation_visibility}
\end{equation}
\end{assumption}

\begin{assumption}
\label{ass:initial_configuration_feasible}
For every sensing edge
\(
(i,j)\in\edges,
\)
the initial relative pose satisfies
\(
\left(\pos_{ij}(0),\quat_i(0)\right)
\in
\mathcal{D}_{ij}.
\)
\end{assumption}

\emph{\textbf{Problem.}}
{Consider the multi-robot system \eqref{eq:underwater_vehicle_model}
with sensing graph satisfying Assumption~\ref{ass:sensing_graph}.
Design decentralized controllers such that:
(i) the actuator constraints \eqref{eq:input_lims} are satisfied;
(ii) the physical constraints \eqref{eq:actual_admissible_domain} are
preserved whenever compatible with the available control authority, while
the conservative barrier domains are adaptively enlarged only within the
corresponding physical domains and recover toward their nominal values when
possible; and
(iii) tracking and formation errors remain bounded for time-varying
references and, under constant references and nominal conditions,
\( \lim_{t\to\infty}\norm{\tilde{\pos}_1(t)}=0,
\)
\(    \lim_{t\to\infty}\norm{\tilde{\pos}_{ij}(t)}=0,
\)
\(
    \forall (i,j)\in\edges.
\)
}

%% file: Sections/control.tex
The controller follows a command-filtered backstepping design. The generalized velocity \(\veltwist_i\) is first treated as a virtual input and a desired velocity \(\veltwist_{i,d}\) is constructed from the gradient of a BLF encoding formation, sensing, and collision-avoidance objectives. Command filtering~\cite{farrell2009command} avoids differentiating \(\veltwist_{i,d}\) and the resulting CLF derivative is affine in the thruster forces, allowing actuator bounds to be enforced directly in a QP with relaxed dissipation. In parallel, auxiliary dynamics adapt the BLF domains between their conservative and physical limits.

\subsection{Recentered Barrier Design}
\label{subsec:recentered_barriers}


We construct a composite BLF~\cite{tee2009barrier} encoding the formation
objective together with the sensing and collision-avoidance constraints.

{
\begin{definition}
    Let \(\mathcal{D}\) be an open set and let
    \(\mathcal{S}\subset\mathcal{D}\) denote a desired equilibrium set. A continuously differentiable function
    \( V:\mathcal{D}\rightarrow\mathbb{R}_{\geq 0} \)
    is a candidate barrier Lyapunov function with respect to
    \(\mathcal{S}\) if
    \(
        V(\boldsymbol{x})=0
    \), $\forall \boldsymbol{x} \in \mathcal{S}$;  
    \(
        V(\boldsymbol{x})>0
    \), 
    \(
        \forall\boldsymbol{x}\in
        \mathcal{D}\setminus\mathcal{S};
    \)
    and
    \( V(\boldsymbol{x})\rightarrow+\infty \) 
    as
    \( \boldsymbol{x}\rightarrow\partial\mathcal{D}. \)
\end{definition}
}

For each constraint in \eqref{eq:desired_formation_visibility}, we define a continuously differentiable scalar function whose positive values characterize admissible poses and whose zero level set defines the conservative boundary. We then associate with it a recentered logarithmic barrier term that diverges at the boundary while having zero value and zero gradient at the desired pose.
Specifically, let
\(
\mathcal{L}\define\{\delta,\Delta,h,v\}
\)
denote the set of constraint indices. For every sensing edge
\(
(i,j)\in\edges,
\)
define the conservative constraint functions
\begin{equation}
\label{eq:barrier_constraint_functions}
\begin{aligned}
    h_{\delta,ij}^{c}
    &\define
    d_{ij}
    -
    d_{\min}^{c}
    &
    h_{\Delta,ij}^{c}
    &\define
    d_{\max}^{c}
    -
    d_{ij},
    \\
    h_{h,ij}^{c}
    &\define
    \left(\alpha_h^{c}\right)^{2}
    -
    \alpha_{h,ij}^{2},
    &
    h_{v,ij}^{c}
    &\define
    \left(\alpha_v^{c}\right)^{2}
    -
    \alpha_{v,ij}^{2}.
\end{aligned}
\end{equation}
The conservative collision-avoidance and sensing constraints are satisfied whenever
\(
h_{\ell,ij}^{c} > 0
\)
for every
\(
\ell\in\mathcal{L}.
\)

\begin{proposition}
\label{prop:positive_camera_depth}
Suppose that
\(
\norm{\pos_{CB_i}}<d_{\min}
\)
and
\(
p_{x,ij}^{C}(0)>0.
\)
If
$d_{ij}>d_{\min}$,
    $\left|\alpha_{h,ij}(t)\right|<1$,
    $\left|\alpha_{v,ij}(t)\right|<1$
for all \(t\geq0\), then
\(
p_{x,ij}^{C}(t)>0
\)
for all \(t\geq0\).
\end{proposition}

{
\begin{proof}
Suppose, by contradiction, that \(p_{x,ij}^{C}\) reaches zero for the first
time at \(t^\star\). Then
\(
p_{x,ij}^{C}(t)\to0^{+}
\)
as \(t\to t^{\star-}\).
Since
\(
|\alpha_{h,ij}|<1
\)
and
\(
|\alpha_{v,ij}|<1
\),
\eqref{eq:normalized_image_coordinates} implies
\(
p_{y,ij}^{C}(t),p_{z,ij}^{C}(t)\to0
\),
and hence
\(
\pos_{ij}^{C}(t)\to\boldsymbol{0}.
\)
From \eqref{eq:relative_target_camera},
\(
\pos_{ij}^{B}(t)\to\pos_{CB_i}
\),
so that
\[
    d_{ij}(t)
    =
    \norm{\pos_{ij}^{B}(t)}
    \to
    \norm{\pos_{CB_i}}
    <
    d_{\min}.
\]
Thus \(d_{ij}(t)<d_{\min}\) for \(t\) sufficiently close to \(t^\star\),
contradicting \(d_{ij}(t)>d_{\min}\). Therefore
\(
p_{x,ij}^{C}(t)>0
\)
for all \(t\geq0\).
\end{proof}
}

To allow for temporary relaxation of these conservative boundaries, we associate with each constraint a dynamic relaxation state
\(
    s_{\ell,i}(t)\in[0,1].
\)
Its dynamics will be introduced in Sec.~\ref{subsec:adaptive_barrier_domain}. For the moment, \(s_{\ell,i}\) is treated as a bounded parameter. 
We define the maximum admissible displacement of each conservative boundary toward its corresponding physical limit as
\begin{equation}
\label{eq:maximum_constraint_relaxations}
\begin{aligned}
    \bar{\rho}_{\delta}
    &\define
    d_{\min}^{c}
    -d_{\min},
    &
    \bar{\rho}_{\Delta}
    &\define
    d_{\max}
    -
    d_{\max}^{c},
    \\
    \bar{\rho}_{h}
    &\define
    1-\left(\alpha_h^{c}\right)^2,
    &
    \bar{\rho}_{v}
    &\define
    1-\left(\alpha_v^{c}\right)^2.
\end{aligned}
\end{equation}
Accordingly, we introduce the shifted constraint functions
\begin{equation}
\label{eq:adaptive_constraint_function}
    h_{\ell,ij}^{a}
    \define
    h_{\ell,ij}^{c}
    +
    \bar{\rho}_{\ell}s_{\ell,i},
    \qquad
    \ell\in\mathcal{L}.
\end{equation}
The corresponding constraint condition is therefore
\(
    h_{\ell,ij}^{a}>0
\) if and only if
\(
    h_{\ell,ij}^{c}
    >
    -\bar{\rho}_{\ell}s_{\ell,i}.
\)
For
\(
    s_{\ell,i}=0,
\)
this coincides with the nominal conservative constraint, whereas
\(
    s_{\ell,i}=1
\)
places the corresponding boundary at its physical limit. 

For each follower \(i\), with unique parent \(j\in\neigh_i\) (see Assumption~\ref{ass:sensing_graph}), define the
desired equilibrium set
\begin{equation}
\label{eq:follower_target_set}
\mathcal{S}_{ij}
\define
\left\{
    (\pos_{ij},\quat_i)
    \;\middle|\;
    \pos_{ij}=\pos_{ij}^{d},\;
    \alpha_{h,ij}=\alpha_{h,ij}^{d},\;
    \alpha_{v,ij}=\alpha_{v,ij}^{d}
\right\}.
\end{equation}
In general, \(\mathcal{S}_{ij}\) does not determine a unique orientation,
since rotations about the relative bearing may leave the regulated image
coordinates unchanged.

For every
\(
    \ell\in\mathcal{L},
\)
let
\(
    h_{\ell,ij}^{d,c}>0
\)
denote the value of
\(
    h_{\ell,ij}^{c}
\)
on the desired set
\(
    \mathcal{S}_{ij}.
\)
For the camera-dependent constraints, the desired image coordinates may, for
instance, be chosen as
\(
    \alpha_{h,ij}^{d}
    =
    \alpha_{v,ij}^{d}
    =
    0,
\)
corresponding to the parent being centered in the image. The positivity of
\(
    h_{\ell,ij}^{d,c}
\)
follows from
Assumption~\ref{ass:desired_formation_feasible}.
Let us define the shifted constraint function
\eqref{eq:adaptive_constraint_function} evaluated at the same desired
pose
\(
    h_{\ell,ij}^{d,a}
    \define
    h_{\ell,ij}^{d,c}
    +
    \bar{\rho}_{\ell}s_{\ell,i}.
\)


Following~\cite{tee2009barrier,restrepo2022robust,wills2002recentred,
feller2015weight}, for \(h,h^d>0\) define the recentered logarithmic
barrier
\begin{equation}
    \overline\beta(h;h^d)
    \define
    -\ln\!\left(\frac{h}{h^d}\right)
    +
    \frac{h}{h^d}
    -1 .
    \label{eq:recentered_barrier}
\end{equation}
It satisfies
\(
\overline\beta(h;h^d)\geq0
\),
with zero value and gradient at \(h=h^d\), while
\(
\overline\beta(h;h^d)\to+\infty
\)
as \(h\to0^+\) (see Fig.~\ref{fig:recentered_barriers}).
Accordingly, for every \(\ell\in\mathcal L\), define
\(
V_{\ell,ij}^{a}
\define
\overline\beta(h_{\ell,ij}^{a};h_{\ell,ij}^{d,a}).
\)
For fixed \(s_{\ell,i}\in[0,1]\), this barrier is minimized at the
desired pose and diverges at the adaptive boundary.

\begin{figure}[t]
    \centering
    \includegraphics[width=1.0\columnwidth]
    {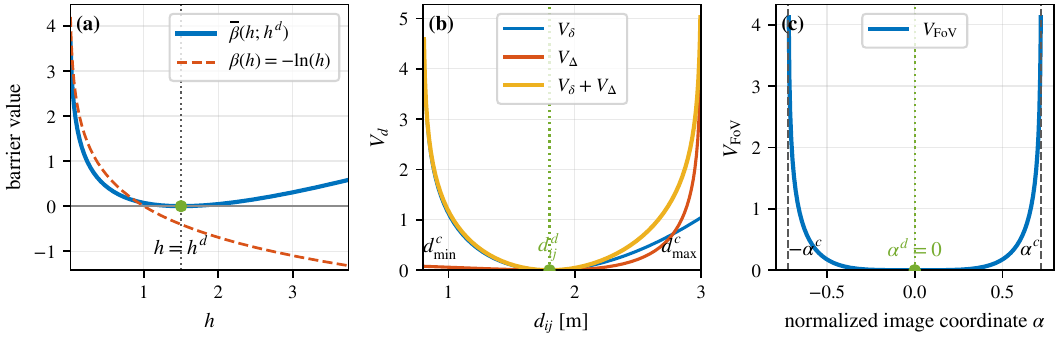}
    \caption{
        Recentered logarithmic barrier functions.
        (a) Comparison between the logarithmic barrier and its recentered
        counterpart. 
        (b) Recentered collision-avoidance and sensing-range barriers as
        functions of the inter-robot distance \(d_{ij}\).
        (c) Recentered field-of-view barrier as a function of the normalized
        image coordinate \(\alpha\).
    }
    \label{fig:recentered_barriers}
\end{figure}

For the leader, the control objective is to track the prescribed absolute
position reference. We therefore define
\begin{equation}\label{eq:leader_potential}
    V_1
    \define
    \frac{1}{2}
    k_{p,1}
    \norm{\tilde{\pos}_{1}}^2,
    \qquad
    k_{p,1}>0.
\end{equation}

For each follower
\(
    i\in\followerset,
\)
with unique parent
\(
    j\in\neigh_i
\), 
the objective is instead
expressed in terms of the desired relative pose. In addition to regulating
the relative position error
\(
    \tilde{\pos}_{ij},
\)
we introduce the image-centering potential
\begin{equation}
\label{eq:image_centering_potential}
    V_{\mathrm{cent},i}
    \define
    \frac{k_{h,i}}{2}
    \left(
        \alpha_{h,ij}-\alpha_{h,ij}^{d}
    \right)^2
    +
    \frac{k_{v,i}}{2}
    \left(
        \alpha_{v,ij}-\alpha_{v,ij}^{d}
    \right)^2,
\end{equation}
where
\(
    k_{h,i},k_{v,i}>0.
\)
The shifted composite potential for follower \(i\) is then
\begin{equation}
\label{eq:follower_adaptive_composite_potential}
    V_i
    \define
    \frac{1}{2}
    k_{p,i}
    \norm{\tilde{\pos}_{ij}}^2
    +
    V_{\mathrm{cent},i}
    +
    \sum_{\ell\in\mathcal{L}}
    \mu_{\ell,i}
    \recenteredBarrier_{\ell,ij}^{a},
\end{equation}
where
\(
    k_{p,i}>0
\)
and
\(
    \mu_{\ell,i}>0
\)
for every
\(
    \ell\in\mathcal{L}.
\)


{
\begin{proposition}[Critical points of the follower potential]
\label{prop:follower_potential_critical_points}
Consider a follower \(i\in\followerset\) with parent \(j\in\neigh_i\).
Suppose that
\(
    \alpha_{h,ij}^{d}
    =
    \alpha_{v,ij}^{d}
    =
    0.
\)
Then the desired set
\(\mathcal{S}_{ij}\) is the unique set of minimizers of \(V_i\).
Moreover, every critical point of \(V_i\) outside
\(\mathcal{S}_{ij}\) is a strict saddle.
\end{proposition}

\begin{proof}
All terms composing \(V_i\) are nonnegative and vanish simultaneously on
\(\mathcal{S}_{ij}\); hence \(\mathcal{S}_{ij}\) is its unique minimum set.

Consider now a critical point of \(V_i\).
For either image coordinate
\(
\alpha\in\{\alpha_{h,ij},\alpha_{v,ij}\}
\), using 
\eqref{eq:recentered_barrier},
\eqref{eq:image_centering_potential}, \eqref{eq:follower_adaptive_composite_potential},
the corresponding image-centering and barrier terms satisfy
\[
    \frac{\partial}{\partial\alpha}
    \left[
        \frac{k}{2}\alpha^2
        +
        \mu\overline{\beta}(h(\alpha);h^d)
    \right]
    =
    \alpha
    \left(
        k+
        \frac{2\mu\alpha^2}{h(\alpha)h^d}
    \right).
\]
Since the term in parentheses is strictly positive inside the barrier
domain, this derivative vanishes if and only if \(\alpha=0\).


Up to the nonsingular scaling diagonal matrix
\(
\diag(\tan(\theta_h^{\max}),\tan(\theta_v^{\max}))^{-1}
\),
the Jacobian
\(\matr J_{\alpha,R}\) of
\(\col(\alpha_{h,ij},\alpha_{v,ij})\)
with respect to the robot orientation coincides with the rotational block
of the standard point-feature interaction matrix~\cite{chaumette2006visual}.
Hence,
\(
\operatorname{rank}(\matr J_{\alpha,R})=2.
\)
Since
\(
\nabla_{\rotMat_i}V_i
=
\matr J_{\alpha,R}^{\top}
\col(
{\partial V_i}/{\partial\alpha_{h,ij}},
{\partial V_i}/{\partial\alpha_{v,ij}})
\),
the full column rank of
\(\matr J_{\alpha,R}^{\top}\) implies that
\(
\nabla_{\rotMat_i}V_i=\zeros{}
\)
only if
\(
{\partial V_i}/{\partial\alpha_{h,ij}}
=
{\partial V_i}/{\partial\alpha_{v,ij}}
=0,
\)
and therefore
\(
\alpha_{h,ij}=\alpha_{v,ij}=0.
\)
%
At these values, the image-dependent terms also have zero translational gradient. It therefore remains to characterize the critical points of the
formation-error and distance-dependent terms. Define
\(
    U_i(d_{ij})
    \define
    \mu_{\delta,i}\recenteredBarrier_{\delta,ij}
    +
    \mu_{\Delta,i}\recenteredBarrier_{\Delta,ij}.
\)
Since
\(
    \partial^2\overline{\beta}/\partial h^2
    =
    1/h^2>0,
\)
the function
\(
    \overline{\beta}(h;h^d)
\)
is strictly convex in \(h\). Moreover, the distance constraint functions
\(
    h_{\delta,ij}(d_{ij})
\)
and
\(
    h_{\Delta,ij}(d_{ij})
\)
are affine functions of \(d_{ij}\). Hence, their compositions
with the recentered barrier are strictly convex
\cite[Section~3.2.2]{boyd2004convex}, and so is their positive weighted sum
\(U_i(d_{ij})\). Since
\(
    \partial U_i/\partial d_{ij}=0
\)
at
\(
    d_{ij}=d_{ij}^{d},
\)
this point is the unique minimizer of \(U_i\).

The resulting translational potential has the same annular critical-point
structure as the constrained formation potential analyzed in
\cite[Appendix~I]{restrepo2022robust}. In particular,
\(
    \pos_{ij}=\pos_{ij}^{d}
\)
is its unique minimum, while there exists a unique critical relative position on the ray opposite to \(\pos_{ij}^{d}\), which is a saddle.

It remains only to verify that the image-dependent terms do not remove the
saddle property. At the undesired translational critical point,
\(
    \pos_{ij}
    =
    -({d_{ij}}/{d_{ij}^{d}})\pos_{ij}^{d}.
\)
Consider a small rotation
\(
    \matr{Q}(\varepsilon)\in\mathrm{SO}(3)
\)
about any axis orthogonal to
\(
    \pos_{ij}^{d}
\),
with
\(
    \matr{Q}(0)=\eye{3}.
\)
Perturb the relative position and the robot orientation simultaneously as
\(
    \pos_{ij}(\varepsilon)
    =
    \matr{Q}(\varepsilon)\pos_{ij},
\) 
\(
    \rotMat_i(\varepsilon)
    =
    \matr{Q}(\varepsilon)\rotMat_i.
\)
Since rotations preserve norms,
\(
    \norm{\pos_{ij}(\varepsilon)}
    =
    d_{ij},
\)
and therefore all distance-dependent barrier terms remain constant.
Moreover,
\(
    \rotMat_i(\varepsilon)^{\top}\pos_{ij}(\varepsilon)
    =
    \rotMat_i^{\top}\pos_{ij},
\)
so the relative position expressed in the body frame, and hence in the
camera frame, is unchanged. Consequently,
\(
    \alpha_{h,ij}
\)
and
\(
    \alpha_{v,ij}
\)
remain constant as well.

Thus, along this variation, only the formation-error term changes. Using
the antipodal relation above,
\[
    \frac{k_{p,i}}{2}
    \norm{
        \pos_{ij}(\varepsilon)-\pos_{ij}^{d}
    }^{2}
    =
    \frac{k_{p,i}}{2}
    \left(
        d_{ij}^{2}
        +
        (d_{ij}^{d})^{2}
        +
        2d_{ij}d_{ij}^{d}\cos\varepsilon
    \right),
\]
and therefore
\(
    \left.
    {\mathrm{d}^{2}V_i}/{\mathrm{d}\varepsilon^{2}}
    \right|_{\varepsilon=0}
    =
    -k_{p,i}d_{ij}d_{ij}^{d}
    <0.
\)
Hence the undesired critical point has a direction of strictly negative
curvature and is therefore a strict saddle. Notice that one rotational degree of freedom remains unconstrained. Hence, for a fixed critical relative position, the corresponding critical orientations form a one-dimensional manifold in \(\mathrm{SO}(3)\).

\end{proof}
}

In the nominal case
\(
    \vect{s}_i=\boldsymbol{0},
\)
boundedness of
\(
    V_i
\)
along a closed-loop trajectory initialized inside the conservative domain
precludes any
\(
    h_{\ell,ij}^{c}
\)
from approaching zero, since the corresponding barrier term diverges at the
boundary. This yields forward invariance of the nominal conservative domain \cite[Lemma 1]{tee2009barrier}.
In the adaptive construction, the relaxation states are instead allowed to
evolve, so that
\(
    h_{\ell,ij}^{c}
\)
may become negative while
\(
    h_{\ell,ij}^{a}>0
\)
is maintained. The auxiliary dynamics governing
\(
    s_{\ell,i}
\)
and ensuring this property are introduced in
Sec.~\ref{subsec:adaptive_barrier_domain}.


\subsection{Command-Filtered Backstepping}
\label{subsec:command_filtered_backstepping}

The potentials $V_i$ introduced in \eqref{eq:leader_potential} and \eqref{eq:follower_adaptive_composite_potential} depend on the robot poses, whereas the control wrench enters the dynamics at the generalized-acceleration level. We therefore employ a backstepping construction in which the generalized velocity is first treated as a virtual control input for the subsystem \eqref{eq:underwater_vehicle_model_kinematics}.


Define the geometric body-frame gradient of the barrier $V_i$ as \(\vect{\zeta}_i 
    \define
    \col\left(
            \vect{\zeta}_{p,i},
            \vect{\zeta}_{R,i}
        \right)\), with 
\begin{equation}
    \vect{\zeta}_{p,i}
    \define
    \rotMat_i^{\top}\nabla_{\pos_i}V_i,
    \qquad
    \vect{\zeta}_{R,i}
    \define
    2
    \left[
        \operatorname{skew}
        \left(
            \rotMat_i^{\top}
            \frac{\partial V_i}{\partial\rotMat_i}
        \right)
    \right]^{\vee}.
    \label{eq:potential_gradients}
\end{equation}

The reference, parent motion, and relaxation-state dynamics contribute to
\(\dot V_i\) through the terms
\begin{equation}
\begin{aligned}
    \veltwist_{i,\mathrm{ff}}
    &\define
    \begin{cases}
        \col\!\left(
            \rotMat_1^{\top}\dot{\pos}_1^{d},
            \boldsymbol{0}
        \right), & i=1,\\
        \boldsymbol{0}, & i\in\followerset,
    \end{cases}
    \\
    \chi_i
    &\define
    \begin{cases}
        0, & i=1,\\
        (\nabla_{\pos_j}V_i)^{\top}\rotMat_j\vel_j,
        & i\in\followerset,
    \end{cases}
    \\
    \psi_i
    &\define
    \begin{cases}
        0, & i=1,\\
        (\nabla_{\vect{s}_i}V_i)^{\top}\dot{\vect{s}}_i,
        & i\in\followerset,
    \end{cases}
\end{aligned}
\label{eq:potential_derivative_auxiliary_terms}
\end{equation}
where \(j\in\neigh_i\) is the parent of follower \(i\).
The properties of \(\psi_i\) induced by the relaxation dynamics are analyzed
in Sec.~\ref{subsec:adaptive_barrier_domain}.

Using the vehicle kinematics \eqref{eq:underwater_vehicle_model_kinematics} and \eqref{eq:potential_gradients}, the total derivative of the pose potential $V_i$ can therefore be written as
{
\begin{equation}
\begin{aligned}
\dot V_i
&=
\nabla_{\pos_i}V_i^\top\rotMat_i(\vel_i - \vel_{i,ff})
\\
&
+
\operatorname{tr}
\left[
\left(
\frac{\partial V_i}{\partial\rotMat_i}
\right)^\top
\rotMat_i[\angvel_i-\angvel_{i,\mathrm{ff}}]_\times
\right]
+
    \chi_i
    +
    \psi_i
\\
&=
\vect{\zeta}_i^\top\left(
        \veltwist_i-\veltwist_{i,\mathrm{ff}}
    \right) +
    \chi_i
    +
    \psi_i,
\end{aligned}
\label{eq:local_potential_derivative}
\end{equation}
where the identity
\(
\operatorname{tr}(\matr A^\top[\vect x]_\times)
=
2[\operatorname{skew}(\matr A)]^{\vee\top}\vect x
\)
has been used.}
The leader feedforward term in \eqref{eq:potential_derivative_auxiliary_terms} is directly available from the reference trajectory. For each follower, however, the available inter-agent information is limited to the relative position of its parent obtained from onboard camera measurements. In particular, the parent velocity is not assumed to be directly measurable, and no inter-robot communication is required. Consequently, \(\chi_i\) cannot be canceled by the controller and it remains as an interconnection term between the subsystems.

Treating
\(
    \veltwist_i
\)
as a virtual control input, define the desired generalized velocity as
\begin{equation}
    \veltwist_{i,d}
    \define
    \veltwist_{i,\mathrm{ff}}
    -
    \matr{K}_{\eta,i}\vect{\zeta}_i,
    \label{eq:desired_generalized_velocity}
\end{equation}
where
\(
    \matr{K}_{\eta,i}
    =
    \matr{K}_{\eta,i}^{\top}
    \succ
    \boldsymbol{0}.
\)
If
\(
    \veltwist_i=\veltwist_{i,d},
\)
then
\(
    \dot V_i
    =
    -
    \vect{\zeta}_{i}^{\top}
    \matr{K}_{\eta,i}
    \vect{\zeta}_{i}
    +
    \chi_i
    +
    \psi_i.
\)
Thus, the virtual control introduces a negative-definite term in \(\vect{\zeta}_i\), while $\chi_i$
    and $\psi_i$ remain as disturbances.

A direct backstepping implementation would require
\(
    \dot{\veltwist}_{i,d},
\)
and hence differentiation of the potential gradient
\(
    \vect{\zeta}_i.
\)
For follower robots, these derivatives involve the parent velocity, which is not directly measured. To avoid the unavailable derivative, $\veltwist_{i,d}$ is passed through the first-order command filter
\begin{equation}
    \dot{\veltwist}_{i,f}
    =
    -
    \matr{\Lambda}_i
    \left(
        \veltwist_{i,f}
        -
        \veltwist_{i,d}
    \right),
    \qquad
    \veltwist_{i,f}(0)
    =
    \veltwist_{i,d}(0),
    \label{eq:generalized_velocity_command_filter}
\end{equation}
where
\(
    \matr{\Lambda}_i
    =
    \matr{\Lambda}_i^{\top}
    \succ
    \boldsymbol{0}.
\)
The derivative
\(
    \dot{\veltwist}_{i,f}
\)
is therefore directly available from \eqref{eq:generalized_velocity_command_filter}, without differentiating $\veltwist_{i,d}$.

Define the command-filter error and generalized-velocity tracking error as
\(
    \tilde{\veltwist}_{i}
    \define
    \veltwist_{i,f}
    -
    \veltwist_{i,d},
\)
and 
\(
    \vect{e}_{\veltwist,i}
    \define
    \veltwist_i
    -
    \veltwist_{i,f}.
\)
It follows from \eqref{eq:desired_generalized_velocity} that
\(
    \veltwist_i
    =
    \veltwist_{i,\mathrm{ff}}
    -
    \matr{K}_{\eta,i}\vect{\zeta}_i
    +
    \tilde{\veltwist}_{i}
    +
    \vect{e}_{\veltwist,i}.
\)
Substitution into \eqref{eq:local_potential_derivative} yields
\begin{equation}
    \dot{V}_i
    =
    -
    \vect{\zeta}_i^{\top}
    \matr{K}_{\eta,i}
    \vect{\zeta}_i
    +
    \vect{\zeta}_i^{\top}
    \tilde{\veltwist}_{i}
    +
    \vect{\zeta}_i^{\top}
    \vect{e}_{\veltwist,i}
    +
    \chi_i
    +
    \psi_i.
    \label{eq:filtered_potential_derivative}
\end{equation}

To backstep through the vehicle dynamics, consider the Lyapunov candidate
\begin{equation}
    W_i
    \define
    V_i
    +
    \frac{1}{2}
    \vect{e}_{\veltwist,i}^{\top}
    \inertiaM_i
    \vect{e}_{\veltwist,i}.
    \label{eq:backstepping_clf}
\end{equation}

For compactness, define
\(
    \vect{h}_i(\pose_i,\veltwist_i)
    \define
    \coriolis_i(\veltwist_i)\veltwist_i
    +
    \damping_i(\veltwist_i)\veltwist_i
    +
    \gravitywrench_i(\pose_i).
    \label{eq:vehicle_drift_compact}
\)
Using the vehicle dynamics
\eqref{eq:underwater_vehicle_model_dynamics}, the generalized-velocity
tracking error satisfies
\begin{equation}
    \inertiaM_i
    \dot{\vect{e}}_{\veltwist,i}
    =
    \matr{B}_i\vect{f}_i
    -
    \vect{h}_i
    -
    \inertiaM_i
    \dot{\veltwist}_{i,f}.
    \label{eq:generalized_velocity_error_dynamics}
\end{equation}

Combining \eqref{eq:filtered_potential_derivative}, \eqref{eq:backstepping_clf} and
\eqref{eq:generalized_velocity_error_dynamics} gives
\begin{equation}
    \dot W_i
    =
    -
    \vect{\zeta}_i^{\top}
    \matr{K}_{\eta,i}
    \vect{\zeta}_i
    +
    \vect{\zeta}_i^{\top}
    \tilde{\veltwist}_i
    +
    a_i
    +
    \vect{b}_i^{\top}\vect{f}_i
    +
    \chi_i
    +
    \psi_i,
    \label{eq:affine_clf_derivative_thrusters}
\end{equation}
where
\begin{equation}
\begin{aligned}
    a_i
    &\define
    \vect{\zeta}_i^{\top}
    \vect{e}_{\veltwist,i}
    -
    \vect{e}_{\veltwist,i}^{\top}
    \left(
        \vect{h}_i
        +
        \inertiaM_i
        \dot{\veltwist}_{i,f}
    \right),
    \\
    \vect{b}_i
    &\define
    \matr{B}_i^{\top}
    \vect{e}_{\veltwist,i}.
\end{aligned}
\label{eq:clf_nominal_drift}
\end{equation}
The locally available term
\(
    a_i+\vect{b}_i^{\top}\vect{f}_i
\)
is affine in the thruster forces and is used in the following subsection to
impose the desired dissipation associated with the generalized-velocity
tracking error. The terms involving
\(
\vect{\zeta}_i
\),
the command-filter error
\(
\tilde{\veltwist}_i
\),
and the interconnection term
\(
\chi_i
\)
are retained explicitly in the resulting CLF derivative.


\subsection{CLF-QP-Based Control Design}
\label{subsec:clf_qp}

The command-filtered backstepping design of
Sec.~\ref{subsec:command_filtered_backstepping} yields a Lyapunov derivative
affine in the thruster forces. We therefore impose the desired dissipation
directly through a constrained optimization problem that accounts for the
actuator limits without canceling nonlinearities.

We impose the relaxed CLF condition
\(
    a_i
    +
    \vect{b}_i^{\top}\vect{f}_i
    \leq
    -
    \vect{e}_{\veltwist,i}^{\top}
    \matr{K}_{e,i}
    \vect{e}_{\veltwist,i}
    +
    \delta_i,
\)
where
\(
    \matr{K}_{e,i}
    =
    \matr{K}_{e,i}^{\top}
    \succ\boldsymbol{0}
\)
and
\(
    \delta_i\geq0
\)
is a CLF relaxation variable.

The following result characterizes the nominal closed-loop behavior when the
CLF condition can be enforced without relaxation and the adaptive-domain
dynamics are inactive.

{
\begin{proposition}[Nominal convergence and robustness]
\label{prop:nominal_convergence_robustness}
Consider a follower \(i\in\followerset\) in the nominal case
\(\vect{s}_i\equiv\boldsymbol{0}\),
and suppose that the CLF condition is satisfied with zero relaxation, i.e.
\begin{equation}
    a_i
    +
    \vect{b}_i^{\top}\vect{f}_i
    \leq
    -
    \vect{e}_{\veltwist,i}^{\top}
    \matr{K}_{e,i}
    \vect{e}_{\veltwist,i}.
    \label{eq:zero_slack_clf_constraint}
\end{equation}
Let
\(
    \mathcal{A}_i
    \define
    \left\{
        (\pose,\vect{e}_{\veltwist,i})
        \;\middle|\;
        \pose\in\mathcal{S}_{ij},
        \;
        \vect{e}_{\veltwist,i}=\boldsymbol{0}
    \right\}.
\)
Suppose that the hypotheses of
Proposition~\ref{prop:follower_potential_critical_points} hold.
If
\(
\tilde{\veltwist}_i=\boldsymbol{0}
\)
and 
$\vel_j = \zeros{}$, then \(\mathcal{A}_i\) is asymptotically stable and almost globally attractive relative to the barrier domain. 
Moreover, there exists a neighborhood of
\(\mathcal{A}_i\) in which the closed-loop subsystem is input-to-state stable (ISS)
with respect to
\(
    \vect{d}_i
    \define
    \col
    \left(
        \tilde{\veltwist}_i,
        \vel_j
    \right).
\)
\end{proposition}

\begin{proof}
Since
\(
{\vect{s}}_i=\boldsymbol{0},
\)
the adaptive-domain contribution satisfies
\(
\psi_i=0.
\)
Substituting
\eqref{eq:zero_slack_clf_constraint}
into
\eqref{eq:affine_clf_derivative_thrusters}
gives
\begin{equation}
\begin{aligned}
    \dot W_i
    \leq\;&
    -
    \vect{\zeta}_i^{\top}
    \matr{K}_{\eta,i}
    \vect{\zeta}_i
    -
    \vect{e}_{\veltwist,i}^{\top}
    \matr{K}_{e,i}
    \vect{e}_{\veltwist,i}
    +
    \vect{\zeta}_i^{\top}
    \tilde{\veltwist}_i
    +
    \chi_i.
    \label{eq:zero_slack_clf_bound}
\end{aligned}
\end{equation}

Consider first the zero-input case
\(
\tilde{\veltwist}_i=\boldsymbol{0}
\)
and
\(
\vel_j=\boldsymbol{0}
\), which implies $\chi_i=0$.
Then $\dot W_i\le 0$.
Since the recentered barriers diverge as the boundary of the barrier domain
is approached, every finite sublevel set of \(W_i\) is bounded away from that
boundary. Consequently, a trajectory initialized in the barrier domain
remains in a compact sublevel set of \(W_i\), i.e. it satisfies the constraints \eqref{eq:conservative_domain_parameters}.
By LaSalle's invariance principle \cite[Thm. 4.4]{khalil2002nonlinear}, every trajectory approaches the largest
invariant subset of
\(
    \left\{
        \vect{\zeta}_i=\boldsymbol{0},
        \;
        \vect{e}_{\veltwist,i}=\boldsymbol{0}
    \right\}.
\)
Proposition~\ref{prop:follower_potential_critical_points} establishes that
the critical set of \(V_i\) consists of the desired minimum set
\(\mathcal{S}_{ij}\) and, possibly, undesired strict-saddle critical points.
Therefore, except for initial conditions belonging to the stable manifolds of the undesired saddle sets, which have zero Lebesgue measure under the standard nondegeneracy conditions~\cite[Thm.~4.1]{hirsch1977invariant}, trajectories converge to \(\mathcal{A}_i\), yielding almost-global attractivity relative to the barrier domain.

We next establish the local robustness property. By the local structure of
\(V_i\) around its minimum set, \(\mathcal{S}_{ij}\) is a nondegenerate
minimum manifold, i.e., the Hessian of \(V_i\) is positive definite in the
directions normal to \(\mathcal{S}_{ij}\). Hence, in a sufficiently small
neighborhood of \(\mathcal{A}_i\), there exist positive constants
\(c_{\zeta,i}\), \(c_{1,i}\), and \(c_{2,i}\) such
that~\cite[Props.~2.3 and~2.8]{rebjock2025fast}
\begin{equation}
\begin{gathered}
    \norm{\vect{\zeta}_i}^{2}
    \geq
    c_{\zeta,i}
    \left|
        \pose_i
    \right|_{\mathcal{S}_{ij}}^{2},
    \quad
    c_{1,i}
    \left|
        \vect{x}_i
    \right|_{\mathcal{A}_i}^{2}
    \leq
    W_i
    \leq
    c_{2,i}
    \left|
        \vect{x}_i
    \right|_{\mathcal{A}_i}^{2}.
    \label{eq:local_clf_bounds}
\end{gathered}
\end{equation}
For a follower,
\(
\norm{\nabla_{\pos_j}V_i}
=
\norm{\vect{\zeta}_{p,i}}
\),
and therefore
\(
    |\chi_i|
    \leq
    \norm{\vect{\zeta}_i}\norm{\vel_j}.
\)
Let
\(
\underline{k}_{\eta,i}
=
\lambda_{\min}(\matr{K}_{\eta,i})
\)
and
\(
\underline{k}_{e,i}
=
\lambda_{\min}(\matr{K}_{e,i})
\).
Using the preceding bound and
\(
\norm{\tilde{\veltwist}_i}+\norm{\vel_j}
\leq
\sqrt{2}\norm{\vect{d}_i},
\)
Young's inequality gives
\begin{equation}
    \dot W_i
    \leq
    -
    \frac{\underline{k}_{\eta,i}}{2}
    \norm{\vect{\zeta}_i}^{2}
    -
    \underline{k}_{e,i}
    \norm{\vect{e}_{\veltwist,i}}^{2}
    +
    \frac{1}{\underline{k}_{\eta,i}}
    \norm{\vect{d}_i}^{2}.
    \label{eq:local_iss_clf_bound}
\end{equation}
Together with \eqref{eq:local_clf_bounds}, this is a local ISS-Lyapunov
inequality \cite[Sec. 4.9]{khalil2002nonlinear} with respect to \(\vect{d}_i\), which proves the claim.

\end{proof}
}

Since a stationary underwater vehicle may require nonzero actuation to balance the restoring wrench, we center the optimization around a nominal trim allocation
\(
\vect{f}_{i,\mathrm{tr}}(\pose_i)
\)
satisfying
\(
    \matr{B}_i
    \vect{f}_{i,\mathrm{tr}}(\pose_i)
    =
    \gravitywrench_i(\pose_i).
\)
For overactuated vehicles, \(\vect{f}_{i,\mathrm{tr}}\) may be chosen as the minimum-norm admissible solution.

At each control instant, robot \(i\) solves
\begin{subequations}
\label{eq:clf_qp}
\begin{align}
    \left(
        \vect{f}_i^{\star},
        \delta_i^{\star}
    \right)
    =
    \argmin_{\vect{f}_i,\delta_i}
    &\quad
    \frac{1}{2}
    \norm{
        \vect{f}_i-\vect{f}_{i,\mathrm{tr}}
    }_{\matr{H}_i}^{2}
    +
    p_{1,i}\delta_i
    +
    \frac{p_{2,i}}{2}\delta_i^{2}
    \label{eq:clf_qp_objective}
    \\
    \mathrm{s.t.}\quad&
    a_i
    +
    \vect{b}_i^{\top}\vect{f}_i
    \leq
    -
    \vect{e}_{\veltwist,i}^{\top}
    \matr{K}_{e,i}
    \vect{e}_{\veltwist,i}
    +
    \delta_i,
    \label{eq:clf_qp_constraint}
    \\
    &
    \underline{\vect{f}}_i
    \leq
    \vect{f}_i
    \leq
    \overline{\vect{f}}_i,
    \label{eq:clf_qp_input_constraint}
    \\
    &
    \delta_i\geq0.
    \label{eq:clf_qp_relaxation_constraint}
\end{align}
\end{subequations}
Here,
\(
\matr{H}_i
=
\matr{H}_i^{\top}
\succ\boldsymbol{0}
\),
\(p_{1,i}\geq0\), and \(p_{2,i}>0\).
The first term penalizes deviations from the trim allocation, while the
linear and quadratic penalties on \(\delta_i\) discourage activation of the
CLF relaxation and penalize large violations.


{
Since \(\matr{B}_i\) has full row rank,
\(
\vect{b}_i=\matr{B}_i^\top\vect{e}_{\veltwist,i}=\boldsymbol{0}
\)
if and only if
\(
\vect{e}_{\veltwist,i}=\boldsymbol{0}.
\)
Hence, without thruster bounds, the CLF constraint is always feasible with
zero relaxation: for
\(
\vect{e}_{\veltwist,i}\neq\boldsymbol{0}
\)
the affine input term can enforce the dissipation condition, while for
\(
\vect{e}_{\veltwist,i}=\boldsymbol{0}
\)
one has
\(
a_i=\vect{b}_i=\boldsymbol{0}.
\)
}


\subsection{Adaptive Barrier-Domain Enlargement}
\label{subsec:adaptive_barrier_domain}

We now specify the dynamics of the relaxation states
\(s_{\ell,i}\in[0,1]\) introduced in
Sec.~\ref{subsec:recentered_barriers}. The conservative domain is enlarged
only when the corresponding constraint approaches its adaptive boundary
\emph{and} the zero-relaxation CLF condition is incompatible with the
actuator limits. 

For each constraint \(\ell\in\mathcal L\), choose
\(
    0<h_\ell^{\mathrm{on}}
    <h_\ell^{\mathrm{off}}
    <h_{\ell,ij}^{d,c}
\)
and define the smooth activation
\begin{equation}
    \sigma_\ell(h)
    \define
    \begin{cases}
        1, & h\leq h_\ell^{\mathrm{on}},\\
        6\xi_\ell^5-15\xi_\ell^4+10\xi_\ell^3,
        & h_\ell^{\mathrm{on}}<h<h_\ell^{\mathrm{off}},\\
        0, & h\geq h_\ell^{\mathrm{off}},
    \end{cases}
    \label{eq:adaptive_smooth_activation}
\end{equation}
with $\xi_\ell
    \define
    ({h_\ell^{\mathrm{off}}-h})/
         ({h_\ell^{\mathrm{off}}-h_\ell^{\mathrm{on}}}).$
Thus, \(\sigma_\ell(h_{\ell,ij}^{a})\) identifies constraints approaching
their adaptive boundary, while \(\sigma_\ell=0\) at the desired
configuration. The initial value \(s_{\ell,i}(0)\in[0,1]\) is selected such
that
\(
h_{\ell,ij}^{a}(0)>0.
\)
Such an initialization always exists for an initial pose strictly inside
the physical domain, since \(h_{\ell,ij}^{a}\) coincides with the
corresponding physical constraint function at \(s_{\ell,i}=1\).

To determine whether enlargement is required, define the minimum CLF
relaxation compatible with the actuator bounds,
\begin{equation}
    \delta_i^{\mathrm{req}}
    \define
    \left[
        a_i
        +
        \vect e_{\veltwist,i}^{\top}
        \matr K_{e,i}
        \vect e_{\veltwist,i}
        +
        \min_{\underline{\vect f}_i
        \leq\vect f_i\leq\overline{\vect f}_i}
        \vect b_i^\top\vect f_i
    \right]_+ .
    \label{eq:required_clf_relaxation}
\end{equation}
By construction,
\(
\delta_i^{\mathrm{req}}=0
\)
if and only if the zero-relaxation CLF constraint is feasible. Since the
actuator constraints are box constraints, the inner minimization is
available in closed form. We define
\(
    \gamma_i
    \define
    {(\delta_i^{\mathrm{req}})^2}/
         ({(\delta_i^{\mathrm{req}})^2+\varepsilon_\delta^2}),
\) with
\(
    \varepsilon_\delta>0.
\)
Hence, \(\gamma_i=0\) whenever the zero-relaxation CLF condition is
feasible. We use \(\delta_i^{\mathrm{req}}\), rather than the optimized QP
slack, since the latter may also reflect the trade-off with control effort.

Since both \(h_{\ell,ij}^{a}\) and \(h_{\ell,ij}^{d,a}\) are shifted by
\(\bar\rho_\ell s_{\ell,i}\),
\begin{equation}
    D_{\ell,i}
    \define
    \frac{\partial V_i}{\partial s_{\ell,i}}
    =
    -
    \mu_{\ell,i}\bar\rho_\ell
    \frac{
        (h_{\ell,ij}^{a}-h_{\ell,ij}^{d,a})^2
    }{
        h_{\ell,ij}^{a}(h_{\ell,ij}^{d,a})^2
    }
    \leq0.
    \label{eq:adaptive_potential_s_derivative}
\end{equation}
Moreover,
\(
D_{\ell,i}\to-\infty
\)
as
\(
h_{\ell,ij}^{a}\to0^+.
\)
The relaxation dynamics are chosen as
\begin{equation}
\begin{aligned}
    \dot s_{\ell,i}
    =
    \Pi_{\leq1}\Big(
        s_{\ell,i},
        &-k_{s,\ell}
        \bigl(1-\sigma_\ell(h_{\ell,ij}^{a})\bigr)s_{\ell,i}
        \\
        &-
        k_{b,\ell}\gamma_i
        \sigma_\ell(h_{\ell,ij}^{a})D_{\ell,i}
    \Big),
\end{aligned}
\label{eq:adaptive_domain_controller}
\end{equation}
where \(k_{s,\ell},k_{b,\ell}>0\) and
\(
\Pi_{\leq1}(s,v)\define
v
\)
for \(s<1\), and
\(
\Pi_{\leq1}(s,v)\define\min\{0,v\}
\)
for \(s=1\).
At \(s_{\ell,i}=0\), the right-hand side of
\eqref{eq:adaptive_domain_controller} is nonnegative, so no lower
projection is required and \(s_{\ell,i}(t)\in[0,1]\).
The enlargement channel is active only when
\(\gamma_i>0\) and \(\sigma_\ell>0\). In particular, when
\(h_{\ell,ij}^{a}\leq h_\ell^{\mathrm{on}}\), one has
\(\sigma_\ell=1\), the recovery term is disabled, and
\(\dot s_{\ell,i}\geq0\); hence the adaptive domain cannot shrink while
the constraint margin is small. Conversely, for
\(h_{\ell,ij}^{a}\geq h_\ell^{\mathrm{off}}\), \(\sigma_\ell=0\) and
\(s_{\ell,i}\) recovers exponentially toward zero.
Finally, if \(\vect s_i(0)=\vect 0\) and the zero-relaxation CLF
condition is feasible, then \(\gamma_i=0\) and
\eqref{eq:adaptive_domain_controller} gives \(\dot{\vect s}_i=\vect 0\),
so the nominal case of Proposition~\ref{prop:nominal_convergence_robustness} is recovered
without additional hypotheses.

\begin{proposition}[Adaptive-domain well-posedness]
\label{prop:adaptive_wellposed}
Consider a follower \(i\) with parent \(j\), with
\(s_{\ell,i}(0)\in[0,1]\) and
\(h_{\ell,ij}^{a}(0)>0\) for every \(\ell\in\mathcal L\).
Let \(T>0\) be finite and suppose that the closed-loop solution exists on
\([0,T)\), \(\vect d_i\) is bounded, and
\(\delta_i\in\mathcal L_1([0,T))\). Then, for every
\(\ell\in\mathcal L\),
\(
s_{\ell,i}(t)\in[0,1],
\) 
\(
\inf_{t\in[0,T)} h_{\ell,ij}^{a}(t)>0 .
\)
\end{proposition}

\begin{proof}
Invariance of \(s_{\ell,i}\in[0,1]\) follows from the sign of
\eqref{eq:adaptive_domain_controller} at \(s_{\ell,i}=0\) and from the
projection at \(s_{\ell,i}=1\).
Since \(s_{\ell,i}\leq1\), positivity of \(h_{\ell,ij}^{a}\) implies
\(
d_{\min}<d_{ij}<d_{\max}
\)
and
\(
|\alpha_{h,ij}|,|\alpha_{v,ij}|<1,
\)
and hence all \(h_{\ell,ij}^{c}\) are bounded on \([0,T)\).

If \(h_{\ell,ij}^{a}\leq h_\ell^{\mathrm{on}}\), then
\(\sigma_\ell=1\) and
\eqref{eq:adaptive_potential_s_derivative}--%
\eqref{eq:adaptive_domain_controller} yield
\(
D_{\ell,i}\dot s_{\ell,i}\leq0.
\)
Otherwise \(h_{\ell,ij}^{a}>h_\ell^{\mathrm{on}}\); since
\(h_{\ell,ij}^{d,a}\geq h_{\ell,ij}^{d,c}>0\), both
\(D_{\ell,i}\) and \(\dot s_{\ell,i}\) are bounded.
Hence
\(
\psi_i=\sum_{\ell\in\mathcal L}D_{\ell,i}\dot s_{\ell,i}
\leq\bar\psi_i
\)
for some finite \(\bar\psi_i\).

Using \eqref{eq:affine_clf_derivative_thrusters},
\eqref{eq:clf_qp_constraint}, and Young's inequality gives
\(
\dot W_i
\leq
c_{d,i}\norm{\vect d_i}^{2}+\delta_i+\bar\psi_i
\)
for some \(c_{d,i}>0\).
Since \(\vect d_i\) is bounded, \(\delta_i\in\mathcal L_1([0,T))\),
and \(T<\infty\), there exists \(\bar W_i<\infty\) such that
\(W_i(t)\leq\bar W_i\) on \([0,T)\).

Finally,
\(
\overline\beta(h;h^d)\geq\ln(h^d/h)-1
\),
\(W_i\geq\mu_{\ell,i}V_{\ell,ij}^{a}\), and
\(h_{\ell,ij}^{d,a}\geq h_{\ell,ij}^{d,c}\) imply
\(
h_{\ell,ij}^{a}(t)
\geq
h_{\ell,ij}^{d,c}
e^{-1-\bar W_i/\mu_{\ell,i}}
>0 ,
\)
which proves the claim.
\end{proof}

When \(\gamma_i>0\), the enlargement action strengthens as
\(h_{\ell,ij}^{a}\to0^+\); when \(\gamma_i=0\), the enlargement is
inactive and the recovery term is disabled in the near-boundary region.
Proposition~\ref{prop:adaptive_wellposed} guarantees that, under its stated
finite-time boundedness conditions, the adaptive boundary remains strictly
outside the state. At \(s_{\ell,i}=1\), the adaptive and physical
boundaries coincide and no further enlargement is possible; preservation
of the physical constraint then necessarily depends on the available
control authority and, for relative constraints, on the parent motion.

%% file: Sections/results.tex
The proposed approach is evaluated through Software-in-the-Loop (SITL)
simulations in Gazebo with PX4 and real-world experiments using BlueROV2
Heavy underwater vehicles. A sequence of piecewise-constant formations is
provided as reference, together with a reference trajectory for the leader.
For the nominal SITL validation, the controller gains are
\(
\matr{K}_{\eta,i}
=
\operatorname{diag}(0.55,0.55,0.55,0.80,0.80,0.80)
\),
\(
\matr{K}_{e,i}
=
0.8\,\matr{M}_i
\),
and
\(
\matr{\Lambda}_{i}
=
\operatorname{diag}(3,3,3,4,4,4)
\).
The barrier weights are
\(
\mu_{\delta,i}
=
\mu_{\Delta,i}
=
0.18
\)
and
\(
\mu_{h,i}
=
\mu_{v,i}
=
0.25
\),
while the QP parameters are
\(
\matr{H}_{i}
=
3.77\times10^{-4}\matr{I}_{8}
\),
\(
p_{1,i}=100
\),
and
\(
p_{2,i}=5\times10^{3}
\).
For the adaptive-domain dynamics, the barrier gain is
\(
k_{b,\ell}=0.20
\),
the recovery gain is
\(
k_{s,\ell}=0.80
\),
and
\(
\epsilon_{\delta}=10^{-3}
\),
for all constraint channels~$\ell$.
The smooth activation thresholds are selected as
\(
h_{\ell,\mathrm{on}}
=
0.10\,h_{\ell}^{d,c}
\)
and
\(
h_{\ell,\mathrm{off}}
=
0.30\,h_{\ell}^{d,c}
\).
%
A formation of five robots is considered within a model of an experimental water tank, shown on the right of Fig.~\ref{fig:SITL_setup}, with the corresponding sensing-graph topology shown on the left.
The SITL setup includes model mismatch and unmodeled dynamics, including imprecise vehicle parameters (buoyancy, added mass and damping) and unmodeled thruster dynamics.  Fig.~\ref{fig:sitl_constraints_error} summarizes the five-robot SITL results, showing the evolution of the sensing constraints together with the corresponding formation errors.
A separate two-robot stress test evaluates the adaptive-domain mechanism under limited control authority. The desired distance is \(2.20\,\mathrm{m}\), inside the conservative bound \(d_{\max}^{c}=2.40\,\mathrm{m}\), and the follower thrust limits are reduced to \(20\%\) of their nominal values. At \(t=60\,\mathrm{s}\), the leader moves away at \(0.45\,\mathrm{m/s}\) for \(5\,\mathrm{s}\), inducing saturation and activating the distance relaxation. Fig.~\ref{fig:sitl_relaxation} shows that the adaptive bound enlarges while the physical constraint remains satisfied.

\begin{figure}[htbp]
    \centering
    \begin{subfigure}[b]{0.48\linewidth}
        \centering
        \includegraphics[width=\linewidth]{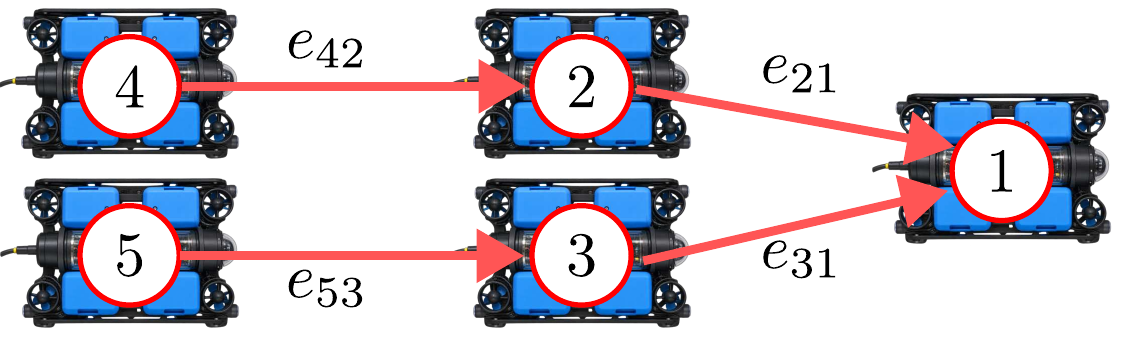}
        \label{fig:sitl_graph}
    \end{subfigure}
    \hfill
    \begin{subfigure}[b]{0.48\linewidth}
        \centering
        \includegraphics[width=\linewidth]{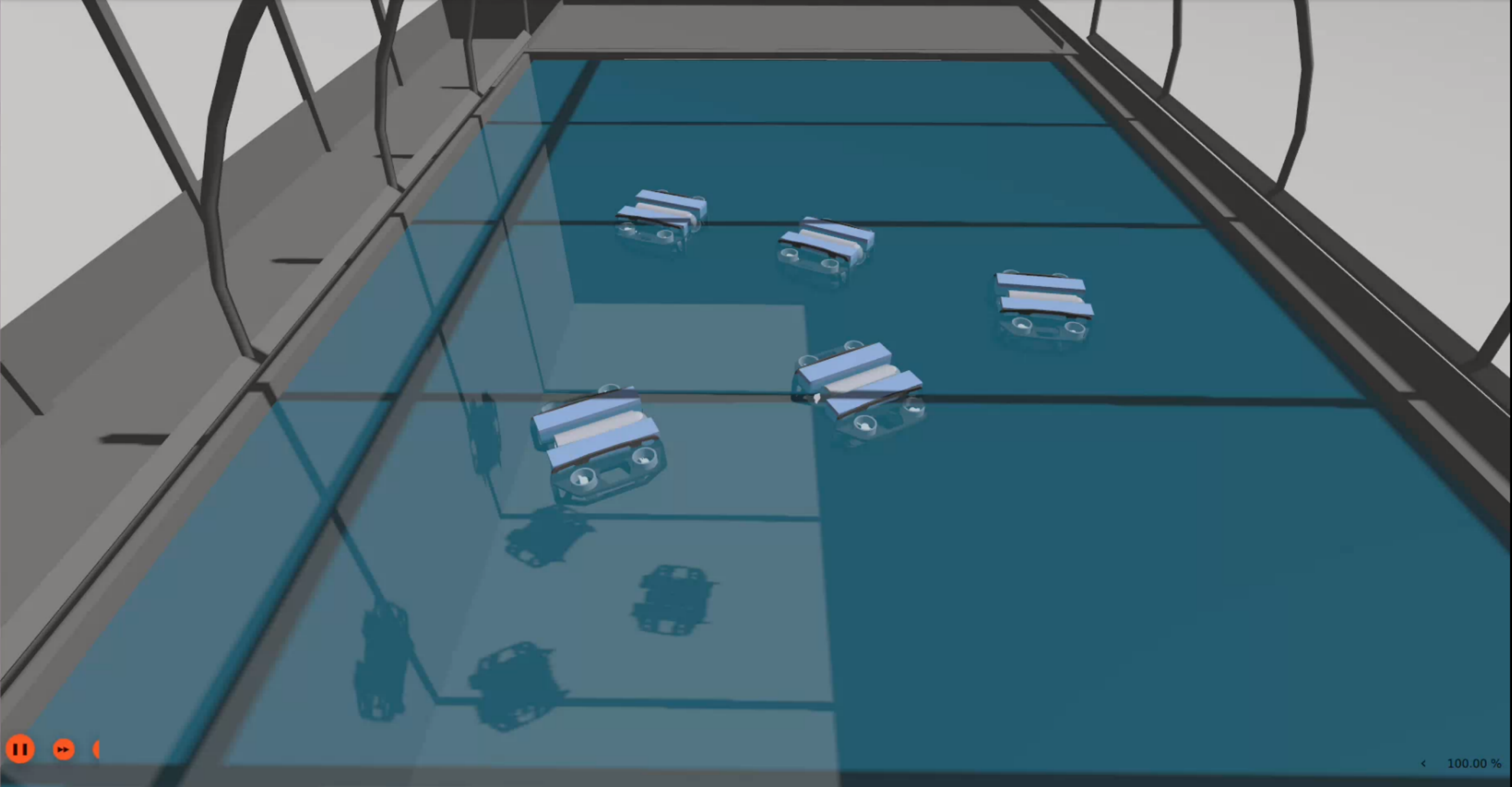}
        \label{fig:gazebo_sim}
    \end{subfigure}
    \caption{
    SITL setup. Left: directed sensing graph. Right: Gazebo simulation setup. 
    }
    \label{fig:SITL_setup}
\end{figure}

\begin{figure}[t]
    \centering
    \includegraphics[width=1.0\columnwidth]{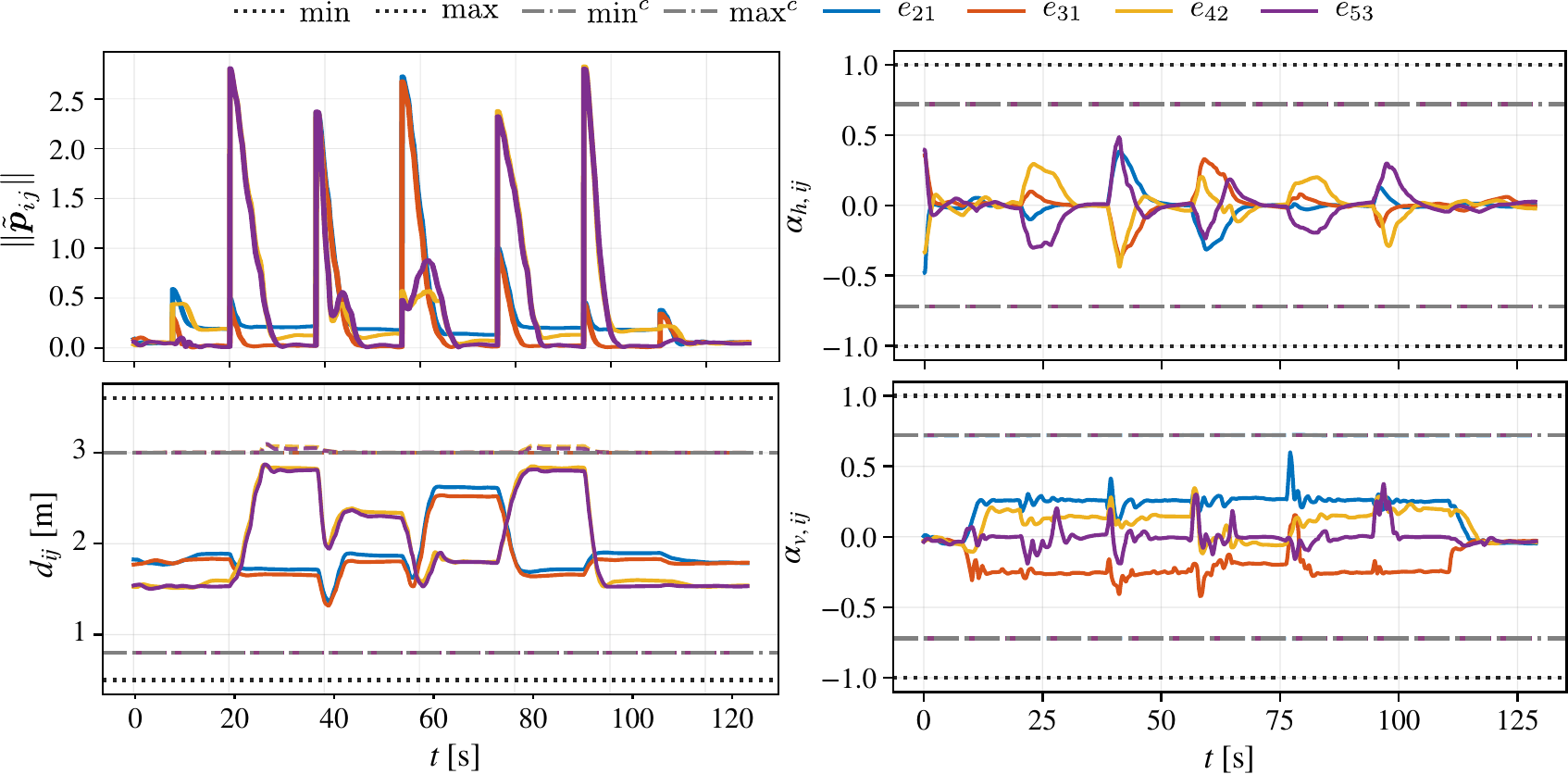}
    \caption{Five-robot SITL results. Top: formation error and horizontal FoV; bottom: distance and vertical FoV. Solid/dashed lines denote measured quantities/adaptive bounds; dotted/dash-dotted lines denote physical/conservative limits.}
    \label{fig:sitl_constraints_error}
\end{figure}

\begin{figure}[t]
    \centering
    \includegraphics[width=1.0\columnwidth]{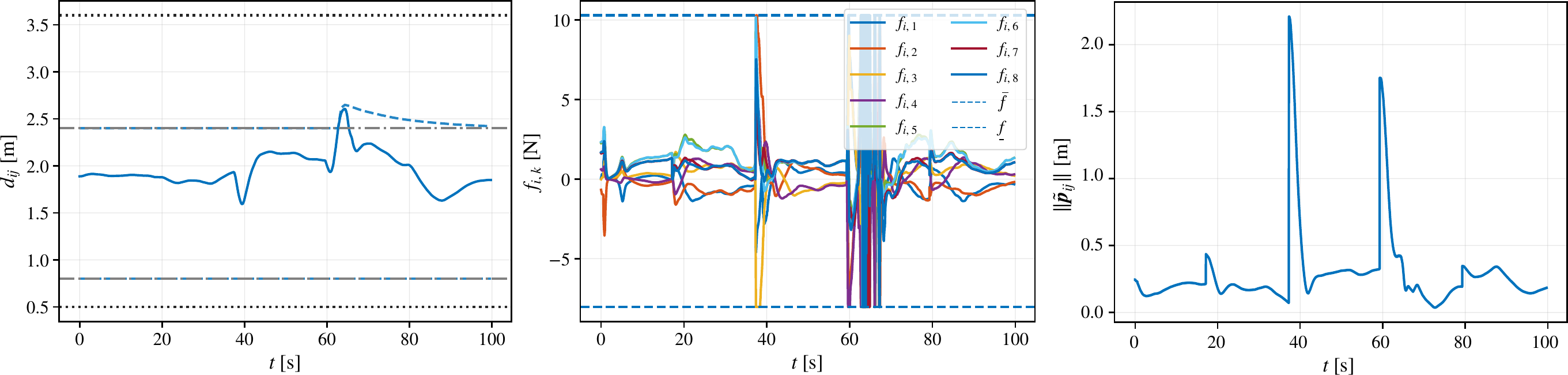}
    \caption{Gazebo/PX4 SITL stress test. Left: inter-robot distance $d_{ij}$ with physical, conservative, and adaptive bounds. Center: follower thruster forces $f_{i,k}$ and limits $\underline{f}$, $\bar{f}$. Right: relative-position error $\|\tilde{\vect{p}}_{ij}\|$. Actuator saturation triggers distance-domain enlargement while the physical constraint remains satisfied.}
    \label{fig:sitl_relaxation}
\end{figure}

%% file: Sections/conclusions.tex
This paper presented a decentralized formation-control framework for underwater multi-robot systems subject to limited sensing and bounded actuation. Constraint satisfaction is encoded through barrier Lyapunov functions whose conservative sensing domains are adaptively enlarged when necessary, while command-filtered backstepping and a thruster-level CLF-QP account explicitly for the vehicle dynamics and actuator limits. The proposed approach is validated through realistic SITL Gazebo simulations 
with underwater vehicles. Future work will focus on lifting the assumption of a common reference frame and on deploying the algorithm fully onboard the robots. 